\documentclass[lettersize,journal]{IEEEtran}
\usepackage{amsmath,amsfonts}
\usepackage{algorithmic}
\usepackage{algorithm}
\usepackage{array}
\usepackage[caption=false,font=normalsize,labelfont=sf,textfont=sf]{subfig}
\usepackage{textcomp}
\usepackage{stfloats}
\usepackage{url}
\usepackage{verbatim}
\usepackage{graphicx}
\usepackage{cite}
\usepackage{color, xcolor}
\usepackage{booktabs}
\usepackage{mathtools}
\usepackage{amsthm}

\usepackage{tabularx} 

\newtheorem{theorem}{Theorem}

\newtheorem{lemma}{Lemma}
\newtheorem{corollary}{Corollary}
\newtheorem{definition}{Definition}

\usepackage{graphicx}

\begin{document}

\title{Simplified PCNet with Robustness}

\author{Bingheng Li*, Xuanting Xie*, Haoxiang Lei, Ruiyi Fang, and Zhao Kang\thanks{*These authors contributed equally.}\thanks{B. Li, X. Xie, H. Lei, Z. Kang are with the School of Computer Science and Engineering, University of Electronic Science and Technology of China, Chengdu, China (e-mail: bingheng86@gmail.com; x624361380@outlook.com; 2021080301020@std.uestc.edu.cn; zkang@uestc.edu.cn;).}\thanks{R. Fang is with western university (email: rfang32@uwo.ca).}}

\markboth{Journal of \LaTeX\ Class Files,~Vol.~14, No.~8, August~2021}%
{Shell \MakeLowercase{\textit{et al.}}: A Sample Article Using IEEEtran.cls for IEEE Journals}


\maketitle

\begin{abstract}
Graph Neural Networks (GNNs) have garnered significant attention for their success in learning the representation of homophilic or heterophilic graphs. However, they cannot generalize well to real-world
graphs with different levels of homophily. In response, the Possion-Charlier Network (PCNet) \cite{li2024pc}, the previous work,  allows graph representation to be learned from heterophily to homophily. Although PCNet alleviates the heterophily issue, there remain some challenges in further improving the efficacy and efficiency. In this paper, we simplify PCNet and enhance its robustness. We first extend the filter order to continuous values and reduce its parameters. Two variants with adaptive neighborhood sizes are implemented. Theoretical analysis shows our model's robustness to graph structure perturbations or adversarial attacks. We validate our approach through semi-supervised learning tasks on various datasets representing both homophilic and heterophilic graphs. 
\end{abstract}

\section{Introduction}
Graph serves as powerful mathematical representations that capture intricate interactions
among entities and has seen a surge in interest due to its ubiquity \cite{wang2017community}. For example, a social network characterizes people's mutual relations; a protein network depicts the chemical compound's linking relationships. The graph usually contains two types of information: node attributes and graph structure. Graph neural networks (GNNs) are dedicated tools for dealing with them and have shown promising results in graph representation learning \cite{GCN, SGC, wei2023stgsa}. The most essential component of GNNs is the well-designed message-passing mechanism, which smooths the signals with neighboring nodes. It is capable of using both the topology structure and the node feature in a coordinated and flexible manner. 

Although GNNs are making an impact in the real world, they have been shown to work well only when the homophily assumption is satisfied, that is, edges are more likely to exist between nodes belonging to the same class \cite{APPNP, li2021beyond}. Real-world data are complex, and heterophilic graphs, that is, connected nodes tend to be of different classes, are also ubiquitous \cite{zhu2020beyond}. The GNN architecture that aggregates information from local neighbors fails to handle heterophily \cite{zhu2020beyond}. Therefore, determining the appropriate neighborhood size for different nodes becomes the core of many research endeavors. 


Many works have been proposed to incorporate multi-hop neighbors to tackle heterophily \cite{EvenNet}. In some mild situations, two-hop neighbors tend to be dominated by homophilic nodes \cite{zhu2020beyond}. Distant neighbors are propagated by applying multiple GNN layers jointly, which faces performance degradation if stacking more than four layers \cite{li2018deeper}. As a result, they can only capture local structures and cannot learn global information, since long-path attachment is ignored at the local level \cite{nam2023global}. In fact, regardless of the number of expanded hops, some neighboring nodes are still missed \cite{Global-hete}. Thus, it is necessary to consider global information. Recently, \cite{grand,grand++} create a topological representation of the whole graph structure using a single global matrix to implicitly perform multiple aggregations. However, they are based on complicated designs to reduce complexity and ignore heterophily. 


\begin{figure}[t]
    \centering
    \includegraphics[width=1.0\linewidth]{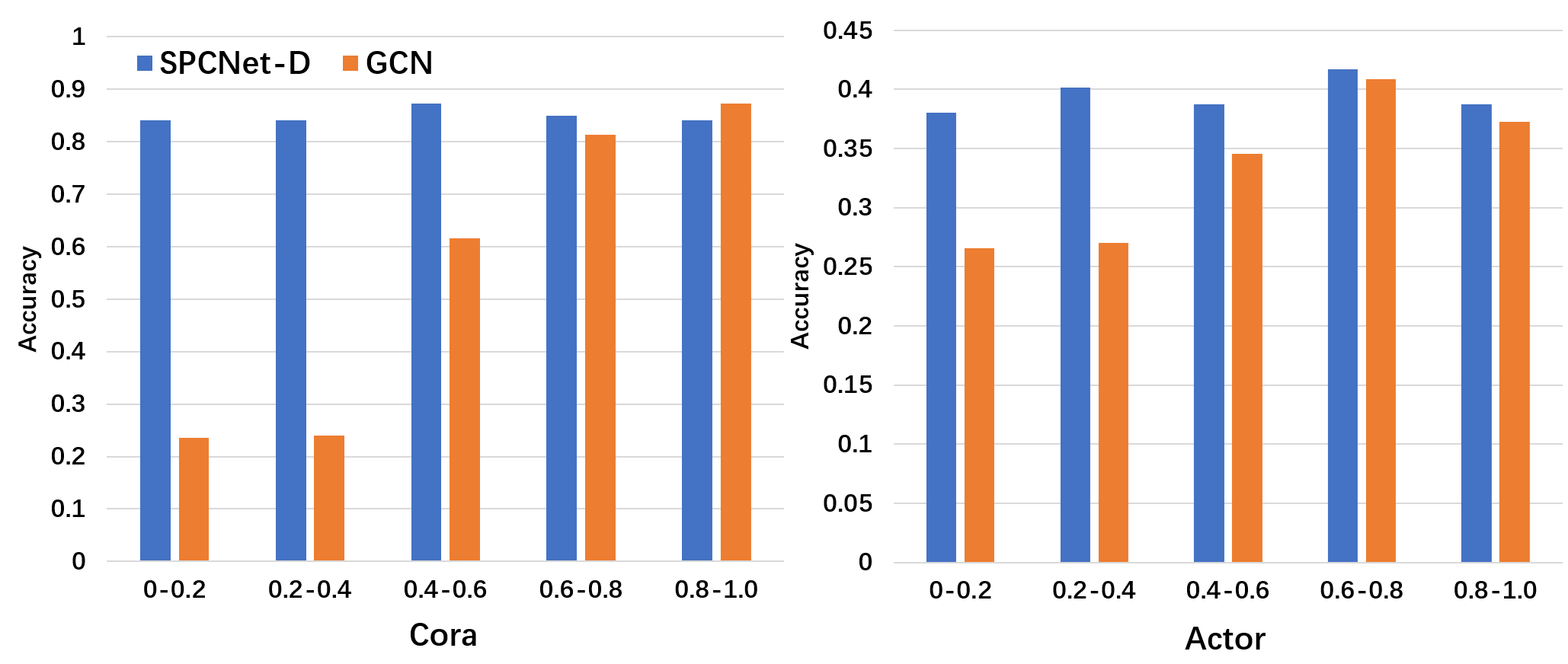}
    \caption{The classification accuracy on nodes with different homophilic degrees. SPCNet-$D$ gives a stable performance, while the performance of GCN varies considerably.}
    \label{figdis2}
\end{figure}

Practical graphs typically exhibit a mixture of homophilic and heterophilic patterns \cite{mao2023demystifying,li2022finding}. As shown in Fig. \ref{figdis2}, the classification accuracy varies dramatically for nodes with different levels of homophily. In our previous work, the Possion-Charlier Network (PCNet) \cite{li2024pc} was developed to extract homophily from heterophilic graphs and vice versa. It derives Possion-Charlier convolution (PC-Conv) based on a twofold filtering mechanism, which performs heterophilic aggregation and homophilic aggregation simultaneously. Although PCNet has shown a pronounced effect, there are still some challenges in further improving the efficacy and efficiency.
(1) \textbf{Lack of Robustness:} Apart from structural disparity, real-world graphs are always disturbed due to measurement errors or other unexpected situations \cite{kenlay2021interpretable}. Therefore, the filter should be designed to be resistant to perturbations in the underlying data, which is ignored by PCNet. Recent research shows that GNNs are vulnerable to adversarial attacks that manipulate graph structures \cite{kenlay2021interpretable}. These attacks drastically degrade GNN performance. Some GNNs can tackle heterophilic graphs but struggle with structural perturbations, while other methods try to enhance the robustness of GNNs but are computationally expensive \cite{entezari2020all, jin2020graph, zhang2020gnnguard, yang2017robust}, prompting further exploration of their robustness \cite{EvenNet, ren2024achieving}. (2) \textbf{Too Many Learnable Parameters:} 
Balancing generalization performance with the abundance of learnable parameters poses a significant challenge in practical applications, particularly in tasks like node classification. Excessive learnable parameters, as seen in PCNet, may lead to overfitting, thereby compromising its ability to generalize effectively. Therefore, simplifying PCNet while maintaining its performance in node classification is a challenging task, offering valuable insights for numerous filter designs.

In light of the above challenges in the PCNet method, we propose a simplified PCNet (SPCNet) to extend the PCNet in a more effective and efficient manner. First, we extend the filter order to continuous values and reduce its parameters, which offers valuable insight for filter designs. Two variants with adaptive neighborhood sizes are implemented. Second, our model is designed with robustness, which addresses the issue of designing filters resistant to structural perturbations and adversarial attacks. A detailed analysis of the robustness is provided. Empirical validation on various datasets representing both homophilic and heterophilic graphs demonstrates the effectiveness of SPCNet in semi-supervised learning tasks. Specifically, several experiments show that SPCNet has clear advantages over PCNet.

\section{Related Work}
\subsection{Spectral GNNs}
Many GNNs have been proposed based on the analysis of the spectral domain. Generally, there are two kinds of them. The first kind of spectral GNN is driven by a fixed filter. GCN \cite{GCN} is based on a first-order Chebyshev polynomial and is considered a low-pass filter. APPNP \cite{APPNP} is based on Personalized PageRank and performs low-pass filtering as well. \cite{GNN-LF-HF} and \cite{xie2023contrastive} combine fixed low- and high-pass filters with graph optimization functions to capture meaningful low- and high-frequency information, respectively. However, fixed filters have limited ability to learn holistic information. The other kind of spectral GNN is driven by a learnable filter. ChebyNet \cite{ChebyNet} is based on Chebyshev polynomials and generalizes convolutional neural networks from regular grids to irregular domains represented by graphs. ARMA \cite{ARMA} is based on an auto-regressive moving average filter and has a more flexible frequency response.
BernNet \cite{BernNet} adopts the Bernstein polynomial approximation to estimate and design these filters with the setting of Bernstein basis coefficients. EvenNet \cite{EvenNet} uses an even-polynomial graph filter and improves generalization between homophilic and heterophilic graphs by ignoring odd-hop neighbors. ChebNetII \cite{Chebyshev2} enhances the original Chebyshev polynomial approximation while reducing overfitting. JacobiConv \cite{JacobiConv} utilizes an orthogonal basis to improve performance while relinquishing nonlinearity in favor of flexibility and adaptability to various weight functions corresponding to the density of graph signals in the spectrum. Recently, rather than using polynomials with a fixed order, OptBasisGNN \cite{OptBasisGNN} orthogonalizes the polynomial basis to learn an optimal one and maximizes convergence speed. However, these methods have the following disadvantages: (1) they require many learnable parameters to approximate target filters, which is time-consuming. (2) They ignore spatial information when designing filters. (3) Most of them ignore the robustness of the model and may be sensitive to noise.


\subsection{Learning on Heterophilic Graphs}
Most GNN-based methods experience performance degradation when dealing with heterophilic graphs. Various strategies have emerged to address this issue by modifying the multi-hop neighbors. MixHop \cite{Mixhop} addresses the limitations of GNN-based methods in capturing neighborhood mixing relationships and proposes to iteratively blend feature representations of neighbors at different distances. GCNII \cite{GCN2} presents a deep GNN model capable of adeptly capturing and leveraging high-order information within graph-structured data. H$_2$GCN \cite{H2GCN} proposes ego- and neighbor-embedding separation, higher-order neighborhoods, and combination of intermediate representations. WRGAT \cite{WRGAT} builds a computation graph using structural equivalences between nodes, leading to increased assortativity and enhanced prediction performance. GPRGNN \cite{GPRGNN} addresses the challenge of efficiently integrating node characteristics and graph structure. It adaptively learns weights to enhance the extraction of node characteristics and topological information, irrespective of node label similarities or differences. GGCN \cite{GGCN} presents two quantitative metrics and suggests two corresponding strategies, structure-based edge correction and feature-based edge correction, which derive signed edge weights from data characteristics. ACM-GCN \cite{ACM-GCN} adaptively exploits aggregation, diversification, and identity channels within each GNN layer to address harmful heterophily and improve GNN performance. LINKX \cite{LINKX} and GloGNN++ \cite{Global-hete} are the most recent work to address heterphilic problem. LINKX \cite{LINKX} proposes a simple method that embeds the feature and graph separately. GloGNN++ \cite{Global-hete} constructs the graph with high-order information to find homophilic nodes in the topology structure. DGCN \cite{pan2023beyond} extracts a homophilic and heterophilic graph from the original structure and performs low- and high-pass filtering.

The above work focuses on either going deep to find homophilic neighbors or mixing local information to enhance performance. However, none of them adaptively select a local neighborhood size. In addition, they do not have a global view of the topology structure.

\subsection{Global GNN}
N-GCN \cite{N-gcn} trains multiple instances of GCNs over node pairs discovered at different distances in random walks and learns a combination of their outputs to optimize the classification objective. DeepGCN \cite{Deepgcns} adapts concepts from deep Convolutional Neural Networks, such as residual connections, dense connections, and dilated convolutions, to GCN architectures. JKNet \cite{JKNet} adaptively leverages different neighborhood ranges for each node to enable better structure-aware representations. RevGNN-Deep \cite{RevGNN-Deep} uses reversible connections combined with deep network architectures, achieving improved memory and parameter efficiency. However, these methods involve either delving deeper to reduce over-smoothing effects or stacking features to expand the multi-hop neighbors. None of these methods specifically concentrates on combining multiple adjacency matrices to create an adjacency matrix with a global view. \cite{grand, grand++} have been proposed to use the exponential function to perform graph convolution with the largest possible multihop neighbors, but they only focus on homophilic graphs and are based on complex designs to reduce complexity. SPGRL\cite{fang2022structure} achieves global structure information by exchange reconstruction method.

\section{Methodology}
\subsection{Preliminaries}\label{preliminaries}

\textbf{Notations}. 
We define an undirected graph as $\mathcal{G}=(\mathcal{V},E)$, where $\mathcal{V} = \left\{v_1, v_2,..., v_m \right\}$ is the node set and $E$ is the edge set with $|E|$ edges. $A\in\mathbb{R}^{m \times m}$ is the adjacency matrix. $A_{i j}= 1$ iff $(v_i, v_j) \in E$, indicates that $v_i$ and $v_j$ share an edge, or else $A_{i j}= 0$. Matrix $D$ is the degree matrix, where $D_{i i}=\sum_j A_{i j}$. Then, the normalized adjacency matrix and the Laplacian matrix are $\tilde{A}=(D+I)^{-\frac{1}{2}} (A+I) (D+I)^{-\frac{1}{2}}$ and $\tilde{L}=I-\tilde{A}$, respectively. $\tilde{L}=U \Lambda U^{\top}$ and the eigenvalue matrix $\Lambda=diag(\tilde{\lambda}_{1},\dots,\tilde{\lambda}_{m})$ indicates the frequency of the graph signal, $U=\{u_{1},\dots,u_{m}\}$ indicates $ \Lambda $'s corresponding frequency component. The filter function of the graph is defined as $h(\tilde{L})$. Furthermore, a perturbed graph with edge omissions or insertions is defined as $\mathcal{G'}$ and its normalized Laplacian matrix is $L_p$. $X\in\mathbb{R}^{m \times d}$ is the node representation matrix where $X_{i:}$ and $X_{:j}$ correspond to the characteristic vector of node $i$ and the $j$-th node attribute, respectively, and $d$ is the dimensionality of the attribute.

\subsection{Cross-receptive Filter}
In this section, we will first introduce our cross-receptive filter and explain why such a filter performs well on different degrees of heterophilic graph from both a local and global point of view. PCNet \cite{li2024pc} gives an accurate approximation to it with Poisson-Charlier polynomials. Then we discuss how to simplify PCNet and achieve an adaptive local neighborhood size in a flexible manner. After that, we conduct an analysis of robustness.

Traditional GNNs have proposed using local-level aggregation over each node, which can smooth graph signals and perform well on homophilic graphs. To be specific, GCN performs aggregation with filter $h(\tilde{L}) = I -  \tilde{L}$ to support learning within local neighborhoods, with a focus on one-hop neighborhoods. To concentrate on local information with more hops, we use the following filter to perform local-level aggregation:
\begin{equation}
         h(\tilde{L}) = (I -  \tilde{L})^k X,
\end{equation}
where $k$ is the filter order. Determining an appropriate neighborhood size for different graphs is difficult. It would be more practical to make $k$ adaptive \cite{zhang2021node}. We will discuss flexible $k$ in the next subsection.

Much work has been proposed to increase the size of the neighborhood to tackle heterophily. Thus, we consider global aggregation to enlarge the receptive field as much as possible. Define a complementary graph $\Bar{\mathcal{G}}$ with adjacency matrix $\Bar{A}=\varphi I-A$, where $\varphi$ is a trade-off parameter to evaluate the self-loop information. Its corresponding Laplacian matrix is $\tilde{L} = (\varphi-2)I+ L $, and its global-level aggregation is $h(\tilde{L})=e^{t \tilde{L}}X$. We can explain the information diffusion of the heterophilic graph by the Taylor expansion.  
\begin{align*}
&h(\tilde{L})=\sum_{r=0}^{\infty} \frac{(t(\varphi-1)I-A)^{2r+1}}{(2r+1)!}+\frac{(A-t(\varphi-1)I)^{2r}}{(2r)!}.
\end{align*}
It is evident that the global aggregation gathers the neighborhood data for even-order neighbors while pushing the odd-order neighbors away, which is compatible with the heterophilic graph's structural characteristic. 

To combine both local and global neighbors, we define our cross-receptive filter $h_{k,t}(\tilde{L})$ as:
\begin{equation}
    Z=h_{k,t}(\tilde{L})X=(I-\tilde{L})^{k}e^{t\tilde{L} }X=U \cdot (I-\Lambda)^{k}e^{t\Lambda} \cdot U^{\top}X,
\end{equation}
where we directly multiply the low-pass and high-pass filters, similar operation in BernNet \cite{BernNet} and JacobiConv \cite{JacobiConv}. With local and global information, our proposed filter can deal with homophilic and heterophilic nodes/edges in the graph.

\subsection{Possion-Charlier Polynomial Approximation}
Using cross-receptive filter directly requires cubic complexity, which is computationally expensive. Following PCNet, we use Possion-Charlier polynomial to approximate it, which is defined as follows \cite{kroeker1977wiener}:
\begin{equation}\label{equiveq}
\begin{aligned}
  (1-\tilde{\lambda})^{\gamma} e^{t \tilde{\lambda}}=\sum_{n=0}^{\infty} \frac{(-\tilde{\lambda})^{n}}{n !} C_{n}(\gamma, t) \quad  \gamma, t \in R,\hspace{.1cm} t>0
\end{aligned}
\end{equation}
where $C_{n}(\gamma, t)=\sum_{k=0}^{n} C_{n}^{k}(-t)^{k} \gamma \cdots(\gamma-n+k+1)$ is the Poisson-Charlier polynomial and has the following recurrence relations: 
\begin{equation}
\begin{aligned}
    &C_{0}(\gamma, t)=1, C_{1}(\gamma, t)=\gamma- t, ...,  C_{n}(\gamma, t)=(\gamma-n-t+1)\\& C_{n-1}(\gamma, t)- (n-1) t C_{n-2}(\gamma, t), n \geq 2 
\end{aligned}
\label{recursion}
\end{equation}

In addition, topology-structure-based aggregation might not be able to learn an effective representation for graphs with a significant degree of heterophily. Inspired by \cite{chen2020simple}, we employ identity mapping to include raw node attributes and guarantee that the representation can at least achieve the same performance as the original data. Finally, our filter is formulated as:
\begin{equation}
\begin{aligned} 
Z &=h_{t}(\tilde{L})X=U(I + h_{k,t}(\tilde{\lambda}))U^{\top} X\\
    &= X + \sum_{n=0}^{N} C_{n}\left(k, t\right) \frac{(-\tilde{L})^{n}}{n !}X,
\end{aligned}
\label{pcconv}
\end{equation}
where $C_{n}\left(k, t\right)$ is calculated using Eq. (\ref{recursion}). Note that PCNet \cite{li2024pc} uses an abundance of learnable parameters in its filter, i.e., $Z^*=\beta_{0} \cdot X +\sum_{k=1}^{K} \beta_{k} \cdot \sum_{n=0}^{N} C_{n}\left(k, t\right) \frac{(-\tilde{L})^{n}}{n !}X$, where $\beta_{0}, \beta_{1},..., \beta_{K}$ are learnable parameters and $K$ is always set to 10. We simplify it with only one term, which is obviously more efficient. To maintain its performance, we propose to select $k$ in a flexible manner.

The preference for $k$ is closely related to the input data, namely the distinct signals present on the underlying graphs. Thus, it would be more practical if $k$ is adaptive. To make SPCNet suitable for different situations, we consider $k$ as a continuous value. Unlike existing work where $k$ is treated as an integer, in our approach, it represents the average value across different filter orders. We propose two ways to obtain it. First, we consider $k$ as a hyperparameter and mark our model as SPCNet-$D$. The second is that we initialize $k$ as 1 and make it learnable, which is marked as SPCNet-$L$.

We introduce the SPCNet-$D$ and SPCNet-$L$ architecture for the node classification task, building upon the Poisson-Charlier polynomial. They consist of two key components: filtering with SPCNet-$D$ or SPCNet-$L$, and a classification component.  
To be specific:
\begin{equation}
\begin{aligned}
&Z=softmax\left((I+h_{k,t}(\tilde{\lambda})) \Theta(X)\right),\\
&\mathcal{L}=-\sum_{l \in \mathcal{Y}_{L}} \sum_{c=1}^{C} Y_{l c} \ln Z_{l c},
\end{aligned}
\end{equation}
where the cross-entropy loss over label matrix $Y$ is used to train the GCN network and $\mathcal{Y}_{L}$ is the labeled node indices.

\subsection{Robustness analysis}

The graph shift operator (GSO) is a linear operator that describes the movement of signals on a graph. The stability of graph filters is primarily investigated by characterizing the magnitude of perturbation resulting from the variations of a GSO under the operator norm. Same as \cite{kenlay2021interpretable}, we define the concept of stability based on relative output distance, which is:
\begin{equation}
\frac{\left\|h(\boldsymbol{\Delta}) \mathbf{x}-h\left(\boldsymbol{\Delta}_p\right) \mathbf{x}\right\|_2}{\|\mathbf{x}\|_2},
\end{equation}
where $h$ is a spectral graph filter, $\mathbf{x}$ is an input graph signal and $\boldsymbol{\Delta}$ is the GSO of the graph $\mathcal{G}$ (similarly $\boldsymbol{\Delta}_p$ is the GSO of $\mathcal{G}^{\prime}$). Without loss of generality, we assume that $\mathbf{x}$ has a unit norm. Then we can bound the above equation to measure the largest possible change made by GSO:

\begin{equation}
\begin{aligned} 
\frac{\left\|h(\boldsymbol{\Delta}) \mathbf{x}-h\left(\boldsymbol{\Delta}_p\right) \mathbf{x}\right\|_2}{\|\mathbf{x}\|_2}
&\leq \max _{\mathbf{x} \neq 0} \frac{\left\|h(\boldsymbol{\Delta}) \mathbf{x}-h\left(\boldsymbol{\Delta}_p\right) \mathbf{x}\right\|_2}{\|\mathbf{x}\|_2} \\
& \stackrel{\text { def }}{=}\left\|h(\boldsymbol{\Delta})-h\left(\boldsymbol{\Delta}_p\right)\right\|_2 .
\end{aligned}
\end{equation}

The error between the matrix is $\left\|L_p-L\right\|_2$. A filter is considered to be linearly stable when it satisfies the following condition.

\begin{definition}
    A spectral graph filter $h(L)$ is linearly stable if its change is limited in constant fast with respect to a type of GSO, i.e.
    \begin{equation}
\left\|h(\boldsymbol{\Delta})-h\left(\boldsymbol{\Delta}_p\right)\right\|_2 \leq \mathbf{C}\|L-L_p\|_2
\end{equation}
\end{definition}


\begin{theorem}
Our method exhibits linear stability when applied to any GSO with a spectrum in the range of $[-1,1]$. Besides, it is more stable than existing polynomial filters.
\end{theorem}

\begin{proof}
Obviously, $\|L\|_2 \leq 1$ and $\left\|L_p\right\|_2 \leq 1$. \cite{levie2019transferability} has proved that $\left\|L_p^k-L^k\right\| \leq k\|L_p-L\|_2$. Combining this with the triangle inequality results in:

$\begin{aligned} 
\left\|h\left(L_p\right)-h(L)\right\|_2
& =\left\|\sum_{n=0}^N C_n \frac{\left(-L_p)^n\right.}{n !}-\sum_{n=0}^N C_n \frac{(-L)^n}{n !}\right\|_2 \\
& \leq \sum_{n=1}^N\left|\frac{(-1)^n C_n}{n !}\right| \cdot\left\|L_p^n-L^n\right\|_2 \\
& \leq \sum_{n=1}^N\left|\frac{C_n}{(n-1) !}\right| \cdot\left\|L_p-L\right\|_2
\end{aligned}$

And the stability constant $\mathbf{C}= \sum_{n=1}^N\left|\frac{C_n}{(n-1) !}\right|$. Thus, our method is robust to graph structure noise. Similarly, a polynomial filter can produce $\mathbf{C}= \sum_{i=0}^N i\left|\theta_i\right|$ \cite{kenlay2021interpretable}. Multiple $\theta_i$ can be largely affected by data and neural networks, whose uncertainty is very high. Thus, our model with certain $\mathbf{C}$ has clear advantages over it. In addition, most polynomial-based methods use the sum of filter orders, making them more sensitive to noise.
\end{proof}

\section{Experiments on Synthetic Graphs}
\begin{figure*}[!htb]
    \centering
    \includegraphics[width=1.\linewidth]{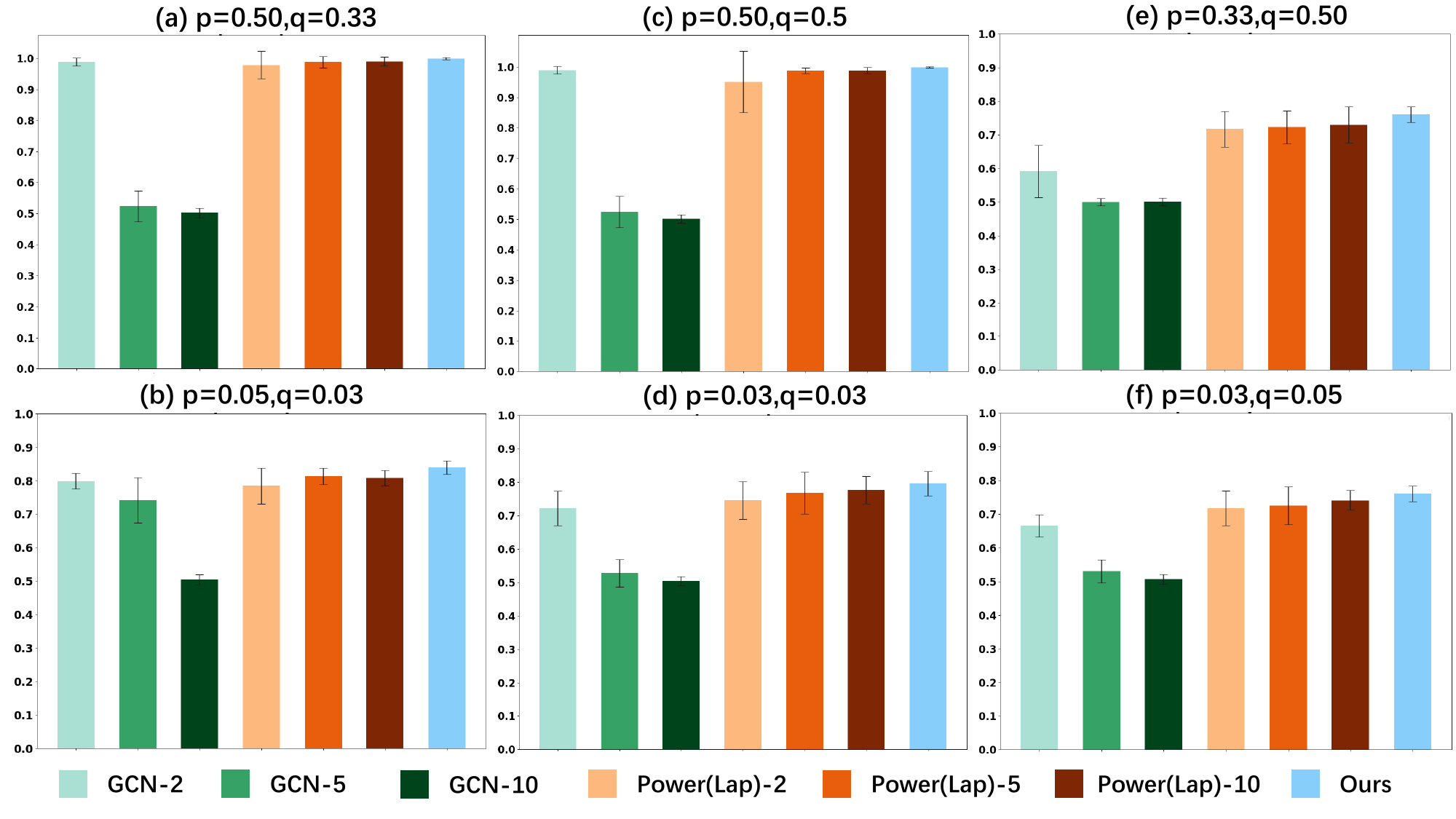}
    \caption{Classification accuracy on synthetic graphs.}
    \label{fig4}
\end{figure*}

To verify that our proposed method can successfully tackle heterophily and capture global information, we first test SPCNet-$D$ on synthetic graphs. For a fair comparison, the experimental settings follow \cite{GloSpectral}, which proposes to construct a two-block symmetric Stochastic Block Model (2B-SBM) graph $A'$. p, q $\in(0,1)$ represent the block connection probability and p$\neq$q. p$>$q indicates homophilic graph and p$<$q denotes heterophilic graph. 

We perform binary node classification on $A'$. It has 500 nodes, and the split is 10/90 train/test. The features $X'$ are sampled from a mixture of two Gaussians with an average $\mu_0$ =$[1, 1]$, $\mu_0=-\mu_1$. The baselines are from the most recent global spectral-inspired GNN \cite{GloSpectral}, including GCN and Power(Lap). Power(Lap) uses the set of representations that includes global spectral information and local neighbor-averaged characteristics with an inception network. We also implement them with different layers, e.g., GCN-2 indicates GCN with 2 layers.

\textbf{Results.} Fig. \ref{fig4} shows the averaged classification accuracy with error bars under different settings of p and q. It can be seen that SPCNet-$D$ achieves a dominant performance in all cases. Specifically, SPCNet-$D$'s high accuracy on dense graphs (Fig. \ref{fig4} (a),(c) and (e)) and sparse graphs (Fig. \ref{fig4} (b), (d) and (f)) is due to its ability to leverage both local and global structural information. Similarly, it can be seen that SPCNet-$D$ can handle heterophlilic graphs (Fig. \ref{fig4} (e) and (f)). GCN-5 and GCN-10 perform very poorly in most cases, which verifies the aforementioned claim that traditional GNN does not go deep. In addition, SPCNet-$D$ and Power(Lap) perform a lot better than GCN on sparse graphs, which is due to the expansion of the multi-hop neighbors. SPCNet-$D$ surpasses Power(Lap) in most cases, which is attributed to our adaptive $k$. 
\begin{table}[h]\tiny
\centering
\caption{Dataset statistics.}\label{tab::datasets}
\begin{center}
\resizebox{.48\textwidth}{!}{
\begin{tabular}{l  l l l l l}
\toprule
Datasets &Nodes  & Edges  & Features  & Classes  & $H(\mathcal{G})$ \\
\midrule
Cora  & 2708 & 5278 & 1433 & 7 & 0.81 \\
Citeseer & 3327 & 4676 & 3703 & 6 & 0.74 \\
Pubmed  & 19,717 & 44,327 & 500 & 3 & 0.80 \\
Actor & 7600 & 26,752 & 931 & 5 & 0.22 \\
Texas  & 183 & 295 & 1703 & 5 & 0.11 \\
Cornell & 183 & 280 & 1703 & 5 & 0.30 \\
Acm  & 3025 & 13,128 & 1870 & 3 & 0.82 \\
Penn94 & 41,554 & 1,362,229 & 5 & 2 & 0.47\\
\bottomrule
\end{tabular}}
\end{center}
\end{table}
\begin{table*}[t]

\centering
\caption{The results of semi-supervised node classification: Mean accuracy (\%) ± 95\% confidence interval. We mark the best performance in \textbf{bold} and the runner-up performance in \textcolor{blue}{blue}.   }
\label{semi}
\resizebox{1.0\textwidth}{!}{
\begin{tabular}{l c c c c c c c c c c c}
\toprule
Dataset & MLP & GCN & ChebNet & ARMA& APPNP & GPRGNN & BernNet & ChebNetII & SPCNet-D & SPCNet-L\\ 
\midrule
Cora  

& 57.17$_{\pm\text{1.34}}$ &79.19$_{\pm\text{1.37}}$  &78.08$_{\pm\text{0.86}}$ &79.14$_{\pm\text{1.07}}$ &82.39$_{\pm\text{0.68}}$ &82.37$_{\pm\text{0.91}}$ &82.17$_{\pm\text{0.86}}$
&{\textcolor{blue}{82.42$_{\pm\text{0.64}}$}}
&80.40$_{\pm\text{0.61}}$
&\textbf{82.64$_{\pm\text{0.20}}$}
\\

CiteSeer 
&56.75$_{\pm\text{1.55}}$ &69.71$_{\pm\text{1.32}}$  &67.87$_{\pm\text{1.49}}$ &69.35$_{\pm\text{1.44}}$ &69.79$_{\pm\text{0.92}}$ &69.22$_{\pm\text{1.27}}$ &69.44$_{\pm\text{0.97}}$
&69.89$_{\pm\text{1.21}}$
&{\textcolor{blue}{71.53$_{\pm\text{0.29}}$}}
&{\textbf{72.27$_{\pm\text{0.24}}$}}
\\
PubMed   
&70.52$_{\pm\text{0.27}}$ &78.81$_{\pm\text{0.24}}$ &73.96$_{\pm\text{0.31}}$ &78.31$_{\pm\text{0.22}}$ 
&79.97$_{\pm\text{0.28}}$
&79.28$_{\pm\text{0.33}}$  &79.48$_{\pm\text{0.41}}$ &79.51$_{\pm\text{1.03}}$
&{\textcolor{blue}{81.32$_{\pm\text{0.83}}$}}
&{\textbf{81.77$_{\pm\text{0.75}}$}}
\\
Texas
&32.42$_{\pm\text{9.91}}$ &34.68$_{\pm\text{9.07}}$   &36.35$_{\pm\text{8.90}}$   &39.65$_{\pm\text{8.09}}$  &34.79$_{\pm\text{10.11}}$ &33.98$_{\pm\text{11.90}}$ &43.01$_{\pm\text{7.45}}$ 
&46.58$_{\pm\text{7.68}}$
&{\textcolor{blue}{66.32$_{\pm\text{7.08}}$}}
&{\textbf{67.61$_{\pm\text{7.42}}$}}
\\

Actor
 &29.75$_{\pm\text{0.95}}$  &22.74$_{\pm\text{2.37}}$  &26.58$_{\pm\text{1.92}}$ &27.02$_{\pm\text{2.31}}$  &29.74$_{\pm\text{1.04}}$  &28.58$_{\pm\text{1.01}}$ 
&29.87$_{\pm\text{0.78}}$  
 &30.18$_{\pm\text{0.81}}$
 &{\textcolor{blue}{35.35$_{\pm\text{0.33}}$}}
  &{\textbf{35.36$_{\pm\text{0.14}}$}}
\\


Cornell
&36.53$_{\pm\text{7.92}}$ &32.36$_{\pm\text{8.55}}$ &28.78$_{\pm\text{4.85}}$  &28.90$_{\pm\text{10.07}}$  &34.85$_{\pm\text{9.71}}$  &38.95$_{\pm\text{12.36}}$ &39.42$_{\pm\text{9.59}}$
&42.19$_{\pm\text{11.61}}$
&{\textcolor{blue}{53.97$_{\pm\text{3.75}}$}}
&{\textbf{54.55$_{\pm\text{2.66}}$}}
\\

\bottomrule
\end{tabular}}
\end{table*}
\section{Experiments on Real Graphs}

\begin{table*}[t]
\centering
\caption{Node classification accuracy of polynomial-based methods. }
\label{acc_res}
\resizebox{0.9\textwidth}{!}{
\begin{tabular}{@{}llll lll@{}}
\toprule
  Method & Cora & Citeseer & Pubmed &Texas  &Actor &Cornell \\ \midrule
MLP 
&76.89$_{\pm\text{0.97}}$
&76.52$_{\pm\text{0.89}}$
&86.14$_{\pm\text{0.25}}$
&86.81$_{\pm\text{2.24}}$

&40.18$_{\pm\text{0.55}}$

 &84.15$_{\pm\text{3.05}}$

\\

GCN        
&87.18$_{\pm\text{1.12}}$      &79.85$_{\pm\text{0.78}}$          &86.79$_{\pm\text{0.31}}$
&76.97$_{\pm\text{3.97}}$ 
&33.26$_{\pm\text{1.15}}$
&65.78$_{\pm\text{4.16}}$

\\

ChebNet    
&87.32$_{\pm\text{0.92}}$      &79.33$_{\pm\text{0.57}}$          &87.82$_{\pm\text{0.24}}$
&86.28$_{\pm\text{2.62}}$

&37.42$_{\pm\text{0.58}}$
&83.91$_{\pm\text{2.17}}$

 \\

ARMA      
&87.13$_{\pm\text{0.80}}$      &80.04$_{\pm\text{0.55}}$          &86.93$_{\pm\text{0.24}}$
&83.97$_{\pm\text{3.77}}$
 
&37.67$_{\pm\text{0.54}}$
&85.62$_{\pm\text{2.13}}$

 \\

APPNP      
&88.16$_{\pm\text{0.74}}$      &80.47$_{\pm\text{0.73}}$
&88.13$_{\pm\text{0.33}}$
&90.64$_{\pm\text{1.70}}$ 
&39.76$_{\pm\text{0.49}}$
&91.52$_{\pm\text{1.81}}$

\\

GCNII        
&88.46$_{\pm\text{0.82}}$      &79.97$_{\pm\text{0.65}}$   
&89.94$_{\pm\text{0.31}}$
&80.46$_{\pm\text{5.91}}$
&36.89$_{\pm\text{0.95}}$
&84.26$_{\pm\text{2.13}}$

\\

TWIRLS 
&88.57$_{\pm\text{0.91}}$
&80.07$_{\pm\text{0.94}}$          &88.87$_{\pm\text{0.43}}$  
&91.31$_{\pm\text{3.36}}$
&38.13$_{\pm\text{0.81}}$
&89.83$_{\pm\text{2.29}}$

 \\

EGNN
&87.47$_{\pm\text{1.33}}$
&80.51$_{\pm\text{0.93}}$         &88.74$_{\pm\text{0.46}}$
&81.34$_{\pm\text{1.56}}$
&35.16$_{\pm\text{0.64}}$
&82.09$_{\pm\text{1.16}}$

\\

PDE-GCN

&88.62$_{\pm\text{1.03}}$
&79.98$_{\pm\text{0.97}}$      &89.92$_{\pm\text{0.38}}$
&93.24$_{\pm\text{2.03}}$

&39.76$_{\pm\text{0.74}}$
&89.73$_{\pm\text{1.35}}$

\\

GPRGNN     
&88.54$_{\pm\text{0.67}}$      &80.13$_{\pm\text{0.84}}$          &88.46$_{\pm\text{0.31}}$
&92.91$_{\pm\text{1.32}}$
&39.91$_{\pm\text{0.62}}$ 
&91.57$_{\pm\text{1.96}}$
 \\

BernNet     
&88.51$_{\pm\text{0.92}}$      &80.08$_{\pm\text{0.75}}$          &88.51$_{\pm\text{0.39}}$
&92.62$_{\pm\text{1.37}}$
&41.71$_{\pm\text{1.12}}$
&92.13$_{\pm\text{1.64}}$
\\

EvenNet
&87.25$_{\pm\text{1.42}}$
&78.65$_{\pm\text{0.96}}$
&89.52$_{\pm\text{0.31}}$
&93.77$_{\pm\text{1.73}}$
&40.48$_{\pm\text{0.62}}$
&92.13$_{\pm\text{1.72}}$
\\

ChebNetII 
&88.71$_{\pm\text{0.93}}$      
&80.53$_{\pm\text{0.79}}$      
&88.93$_{\pm\text{0.29}}$
&93.28$_{\pm\text{1.47}}$
 &\textcolor{blue}{41.75$_{\pm\text{1.07}}$}

&92.30$_{\pm\text{1.48}}$

\\ 
JacobiConv   
&88.98$_{\pm\text{0.46}}$
&80.78$_{\pm\text{0.79}}$
&89.62$_{\pm\text{0.41}}$
&93.44$_{\pm\text{2.13}}$
&41.17$_{\pm\text{0.64}}$

&92.95$_{\pm\text{2.46}}$
\\
OptBasisGNN
&87.00$_{\pm\text{1.55}}$
&80.58$_{\pm\text{0.82}}$
&{\textcolor{blue}{90.30$_{\pm\text{0.19}}$}}
&93.48$_{\pm\text{1.68}}$
&{\textbf{42.39$_{\pm\text{0.52}}$}}

&92.99$_{\pm\text{1.85}}$

\\ \midrule

SPCNet-$D$

  &{\textcolor{blue}{89.34$_{\pm\text{0.84}}$}}
 &\textcolor{blue}{82.16$_{\pm\text{0.57}}$}
 &89.57$_{\pm\text{0.57}}$
  &{\textcolor{blue}{93.93$_{\pm\text{1.28}}$}}
  &40.89$_{\pm\text{0.61}}$
 &{\textbf{94.25$_{\pm\text{1.92}}$}}

\\
SPCNet-$L$

  &{\textbf{89.97$_{\pm\text{0.74}}$}}
 &{\textbf{82.65$_{\pm\text{0.75}}$}}
 &{\textbf{91.82$_{\pm\text{0.36}}$}}
  &{\textbf{94.26$_{\pm\text{1.31}}$}}
  &40.23$_{\pm\text{0.70}}$
 &{\textcolor{blue}{93.62$_{\pm\text{2.13}}$}}

\\
\bottomrule
\end{tabular}
\label{tb:full}}
\end{table*}
\begin{table*}[t]
\centering
\caption{Node classification accuracy (\%) of heterophilic methods.}

\label{tab:result_heter}
\resizebox{1.0\textwidth}{!}{
\begin{tabular}{@{}lllll lll@{}}
\toprule
Dataset & Cora & Citeseer & Pubmed & Texas & Actor & Cornell &Penn94\\ \midrule
     MLP     
     & 75.69$_{\pm\text{2.00}} $   
     & 74.02$_{\pm\text{1.90}} $  
     & 87.16$_{\pm\text{0.37}} $                
     & 80.81$_{\pm\text{4.75}} $       
     & 36.53$_{\pm\text{0.70}} $  
     & 81.89$_{\pm\text{6.40}} $
     &73.61$_{\pm\text{0.40}} $

 \\
     GCN         
     & 86.98$_{\pm\text{1.27}} $ 
     & 76.50$_{\pm\text{1.36}} $    
     & 88.42$_{\pm\text{0.50}} $            
     & 55.14$_{\pm\text{5.16}} $          
     & 27.32$_{\pm\text{1.10}} $ 
     & 60.54$_{\pm\text{5.30}} $ 
    & 82.47$_{\pm\text{0.27}} $
\\
     GAT       
     & 87.30$_{\pm\text{1.10}} $ 
     & 76.55$_{\pm\text{1.23}} $ 
     & 86.33$_{\pm\text{0.48}} $             
     & 52.16$_{\pm\text{6.63}} $          
     & 27.44$_{\pm\text{0.89}} $ 
     & 61.89$_{\pm\text{5.05}} $  
     & 81.53$_{\pm\text{0.55}} $

 \\    
      MixHop      
      & 87.61$_{\pm\text{0.85}} $ 
      & 76.26$_{\pm\text{1.33}} $ 
      & 85.31$_{\pm\text{0.61}} $          
      & 77.84$_{\pm\text{7.73}} $         
      & 32.22$_{\pm\text{2.34}} $ 
      & 73.51$_{\pm\text{6.34}} $ 
      & 83.47$_{\pm\text{0.71}} $


  

\\
   
     GCNII     
     & 88.37$ _{\pm\text{1.25}}$ 
     & 77.33$ _{\pm\text{1.48}}$
     & \textbf{90.15$ _{\pm\text{0.43}}$ }       
     & 77.57 $_{\pm\text{3.83}} $          
     & 37.44$ _{\pm\text{1.30}} $  
     & 77.86$ _{\pm\text{3.79}} $
     & 82.92$_{\pm\text{0.59}} $

 \\
    H$_2$GCN     
    & 87.87$_{\pm\text{1.20}} $ 
    & 77.11$_{\pm\text{1.57}} $  
    & 89.49$_{\pm\text{0.38}} $           
    & 84.86$_{\pm\text{7.23}} $          
    & 35.70$_{\pm\text{1.00}} $  
    & 82.70$_{\pm\text{5.28}} $ 
    & 81.31$_{\pm\text{0.60}} $

 \\ 
    WRGAT     
    & 88.20$_{\pm\text{2.26}} $ 
    & 76.81$_{\pm\text{1.89}}  $  
    & 88.52$_{\pm\text{0.92}} $          
    & 83.62$_{\pm\text{5.50}}  $         
    & 36.53$_{\pm\text{0.77}}  $
    & 81.62$_{\pm\text{3.90}} $  
   & 74.32$_{\pm\text{0.53}} $

 \\ 
    GPRGNN     
    & 87.95$_{\pm\text{1.18}} $ 
    & 77.13$_{\pm\text{1.67}} $ 
    & 87.54$_{\pm\text{0.38}} $          
    & 78.38$_{\pm\text{4.36}} $        
    & 34.63$_{\pm\text{1.22}} $  
    & 80.27$_{\pm\text{8.11}} $ 
    & 81.38$_{\pm\text{0.16}} $

  \\ 
    GGCN      
    & 87.95$_{\pm\text{1.05}} $ 
    & 77.14$_{\pm\text{1.45}} $
    & 89.15$_{\pm\text{0.37}} $            
    & 84.86$_{\pm\text{4.55}} $        
    & 37.54$_{\pm\text{1.56}} $
    & \textcolor{blue}{85.68$_{\pm\text{6.63}} $}
    & OOM

  \\ 
       ACM-GCN          
        & 87.91$_{\pm\text{0.95}} $
        & 77.32$_{\pm\text{1.70}}  $ 
        & 90.00$_{\pm\text{0.52}} $
        & \textcolor{blue}{87.84 $_{\pm\text{4.40}}$ } 
        & 36.28$_{\pm\text{1.09}}  $   
        & 85.14$_{\pm\text{6.07}}  $ 
        & 82.52$_{\pm\text{0.96}} $

   \\  
       LINKX      
       & 84.64$_{\pm\text{1.13}} $ 
       & 73.19$_{\pm\text{0.99}}  $ 
       & 87.86$_{\pm\text{0.77}} $          
       & 74.60$_{\pm\text{8.37}} $        

       & 36.10$_{\pm\text{1.55}} $ 

       & 77.84$_{\pm\text{5.81}}$ 
       & 84.71$_{\pm\text{0.52}} $

  \\
      GloGNN++   
      & 88.33$_{\pm\text{1.09}} $
      & 77.22$_{\pm\text{1.78}} $ 
      & 89.24$_{\pm\text{0.39}} $ 
      & 84.05$_{\pm\text{4.90}}$       
      & 37.70$_{\pm\text{1.40}}$

      & \textbf{85.95$_{\pm\text{5.10}}$ } 
       &\textbf{85.74$_{\pm\text{0.42}}$}

   \\  \midrule
       SPCNet-$D$   
        &\textbf{88.89$_{\pm\text{0.80}}$}
        &\textcolor{blue}{77.69$_{\pm\text{0.92}}$}
        &89.87$_{\pm\text{0.28}}$  
        &\textcolor{blue}{87.84$_{\pm\text{2.43}}$}
        &\textbf{37.83$_{\pm\text{0.72}}$}
        &80.81$_{\pm\text{2.98}}$
        &\textcolor{blue}{84.75$_{\pm\text{0.27}}$}
\\

       SPCNet-$L$   
        & \textcolor{blue}{88.79$_{\pm\text{0.82}}$ } 
        & \textbf{77.84$_{\pm\text{1.02}}$  }
        & \textcolor{blue}{90.04$_{\pm\text{0.32}}$  }
        & \textbf{88.10$_{\pm\text{2.97}} $}  

        & \textcolor{blue}{37.76$_{\pm\text{0.76}}$  }

        & 81.62$_{\pm\text{3.24}}$
        &82.69$_{\pm\text{0.22}}$

\\
 
\bottomrule

\end{tabular}}
\end{table*}
\subsection{Datasets}
To ensure fairness, we choose benchmark datasets commonly employed in related research. Regarding the homophilic graph, we choose three well-established citation graphs: Cora, CiteSeer, and PubMed \cite{Cora}. Concerning the heterophilic graph, we choose three datasets, which encompass the WebKB3 Webpage graphs Texas, Cornell \footnote{http://www.cs.cmu.edu/afs/cs.cmu.edu/project/theo-11/www/wwkb}, as well as the Actor co-occurrence graph \cite{Gemo}. Furthermore, we choose Penn94 from Facebook 100 \footnote{https://archive.org/details/oxford-2005-facebook-matrix} to serve as a representative social network dataset \cite{LINKX}. The calculation of homophily $H(\mathcal{G})$ follows \cite{H2GCN}. A high value of $H(\mathcal{G})$ indicates high homophily. The statistics of these datasets are summarized in Table \ref{tab::datasets}.

In semi-supervised node classification, we divide the dataset using random seeds according to \cite{Chebyshev2}. Specifically, we choose 20 nodes from each class in three homophilic datasets (Cora, Citeser, and Pubmed) for training, allocate 500 nodes for validation, and reserve 1000 nodes for testing. Regarding the three heterophilic datasets (Texas, Actor, Cornell), we use sparse splitting, allocating 2.5\% for training, 2.5\% for validation, and 95\% for testing. We also adopt two popular partitioning methods that have been widely applied in node classification: random seed splits and fixed seed splits. The first method, introduced by \cite{GPRGNN}, uses random seeds for the split and finds application in spectral GNNs, PDE GNNs, traditional methods, and more. The second method employs the fixed split as provided in \cite{Global-hete}, a split commonly used in heterophilic methods. To ensure a fair comparison, we assess the performance of SPCNet under these different settings. 

For node classification using polynomial-based methods, we employ 10 random seeds, following \cite{GPRGNN,JacobiConv}, to partition the dataset into training/validation/test sets with ratios of 60/\%20\%/20\%. In the case of node classification employing heterophilic GNNs, we employ the predetermined split from \cite{Global-hete} to partition the dataset into training/validation/test sets at a ratio of 48\%/32\%/20\%.

\subsection{Baselines}
For polynomial-based methods, the selected methods are as follows. OptBasisGNN \cite{OptBasisGNN}, JacobiConv \cite{JacobiConv}, and EvenNet \cite{EvenNet} are the newest polynomial-based methods. In addition, we incorporate four competitive baselines for node classification: GCNII \cite{GCN2}, TWIRLS \cite{TWIRLS}, EGNN \cite{EGNN}, and PDE-GCN \cite{PDEGCN}. These methods comprehensively take advantage of topological information from various angles, encompassing GNN architecture, energy functions, and PDE GNNs. For methods dealing with heterophily, we select the following baselines: (1) spatial heterophilic GNNs: H$_2$GCN \cite{H2GCN}, WRGAT \cite{WRGAT},  GloGNN++ \cite{Global-hete}; (2) GNN with filterbank: GPRGNN \cite{GPRGNN} and ACM-GCN \cite{ACM-GCN}; (3) MLP-based methods: LINKX \cite{LINKX}; (4) scalable heterophilic GNNs: GCNII \cite{GCN2} and GGCN \cite{GGCN}.

\begin{table*}[t]
\centering
\caption{The mean accuracy (\%) of node classification under 5 different splits on robust test against Meta attack and MinMax attack with perturb ratio 0.20.}
\label{rt20}
\resizebox{0.75\textwidth}{!}{
\begin{tabular}{lcccccc}
\toprule Dataset & Meta-Cora & Meta-Citeseer & Meta-Acm & MinMax-Cora & MinMax-Citeseer & MinMax-Acm \\
\midrule MLP & 58.60 & 62.93 & 85.74 & 59.81 & 63.72 & 85.66 \\
GCN & 63.76 & 61.98 & 68.29 & 69.21 & 68.02 & 69.37 \\
GAT & 66.51 & 63.66 & 68.50 & 69.50 & 67.04 & 69.26 \\
GCNII & 66.57 & 64.23 & 78.53 & 73.01 & 72.26 & 82.90 \\
H2GCN & 71.62 & 67.26 & 83.75 & 66.76 & 69.66 & 84.84 \\
FAGCN & 72.14 & 66.59 & 85.93 & 64.90 & 66.33 & 81.49 \\
GPRGNN & 76.27 & 69.63 & 88.79 & 77.18 & 72.81 & 88.24 \\
\midrule RobustGCN & 60.38 & 60.44 & 62.29 & 68.53 & 63.16 & 61.60 \\
GNN-SVD & 64.83 & 64.98 & 84.55 & 66.33 & 64.97 & 81.08 \\
GNN-Jaccard & 68.30 & 63.40 & 67.81 & 72.98 & 68.43 & 69.03 \\
GNNGuard & 75.98 & 68.57 & 62.19 & 73.23 & 66.14 & 66.15 \\
ProGNN & 75.25 & 68.15 & 83.99 &77.91 & 72.26 & 73.51 \\
 PCNet &77.65 &\textcolor{blue}{72.80} &81.27 &77.32 &70.57 &82.71 \\
EvenNet & \textcolor{blue}{77.74} & 71.03 & \textcolor{blue}{89.78} & \textcolor{blue}{78.40} & \textbf{73.51} & \textcolor{blue}{89.80} \\
\midrule SPCNet-$D$ &\textbf{81.04} &\textbf{74.80} &\textbf{91.27} &\textbf{80.48} &\textcolor{blue}{73.50} &\textbf{89.97}\\
\bottomrule

\end{tabular}}
\end{table*}

\begin{figure*}[t]
\centering
  \includegraphics[width=1.0\textwidth]{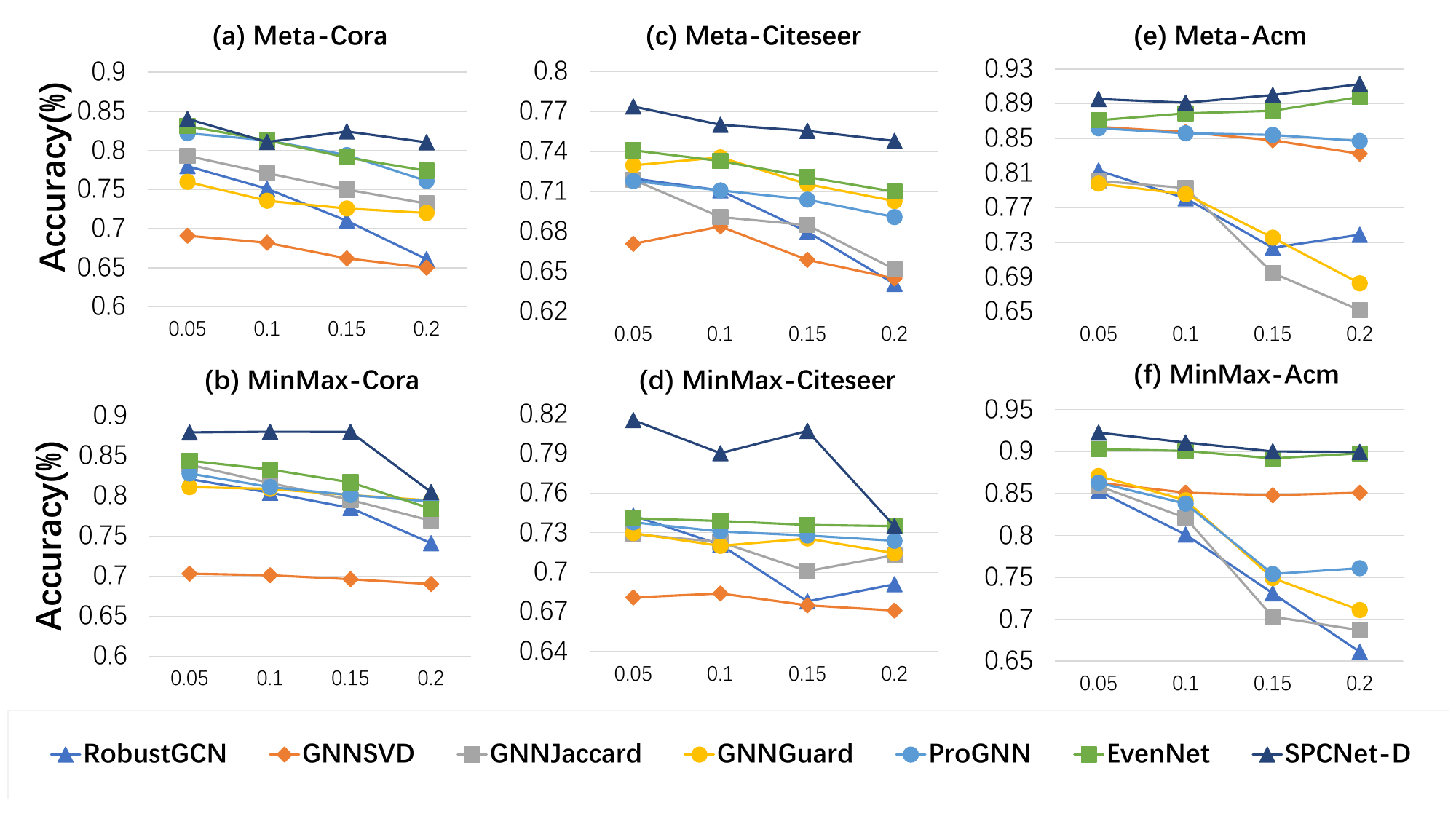}
  \caption{The average node classification accuracy (\%) with different perturb ratios.}
  \label{rt}
\end{figure*}

\begin{table*}[h]
\centering
\caption{ Ablation study on SPCNet.}
\resizebox{0.75\textwidth}{!}{
\begin{tabular}{lcccccc}
\toprule Dataset & Cora & Citeseer & Pubmed & Texas & Actor & Cornell \\
\midrule
SPCNet-GNN & 88.58$_{\pm\text{0.77}}$ & 77.40$_{\pm\text{0.96}}$ & 89.64$_{\pm\text{0.36}}$ & 87.16$_{\pm\text{2.84}}$ & 37.41$_{\pm\text{0.76}}$ & 79.21$_{\pm\text{3.13}}$ \\
SPCNet-Lin & 88.36$_{\pm\text{0.84}}$ & 76.73$_{\pm\text{1.36}}$ & 89.06$_{\pm\text{0.32}}$ & 87.11$_{\pm\text{3.93}}$ & 37.27$_{\pm\text{0.74}}$ & 80.00$_{\pm\text{3.01}}$ \\
SPCNet-$D$ & \textbf{88.89$_{\pm\text{0.80}}$} & \textbf{77.69$_{\pm\text{0.92}}$} & \textbf{89.87$_{\pm\text{0.28}}$} & \textbf{87.84$_{\pm\text{2.43}}$} & \textbf{37.83$_{\pm\text{0.72}}$} & \textbf{80.81$_{\pm\text{2.98}}$} \\
\bottomrule
\end{tabular}
\label{tab::abl1}
}
\end{table*}

\begin{table*}[h]
\centering
\caption{  Results of node classification on PCNet and SPCNet with sparse split.}
\resizebox{.75\textwidth}{!}{
\begin{tabular}{lcccccc}
\toprule Dataset & Cora & Citeseer & Pubmed & Texas & Actor & Cornell \\
\midrule
PCNet & \textbf{82.81$_{\pm\text{0.50}}$} & 69.92$_{\pm\text{0.70}}$ & 80.01$_{\pm\text{0.88}}$ & 64.56$_{\pm\text{1.84}}$ & 33.56$_{\pm\text{0.40}}$ & 52.08$_{\pm\text{4.45}}$ \\
SPCNet-$D$ & 80.40$_{\pm\text{0.61}}$ & \textbf{71.53$_{\pm\text{0.29}}$} & \textbf{81.32$_{\pm\text{0.83}}$} & \textbf{66.32$_{\pm\text{7.08}}$} & \textbf{35.35$_{\pm\text{0.33}}$} & \textbf{53.97$_{\pm\text{3.75}}$} \\
\bottomrule
\label{semi2}
\end{tabular}
}
\end{table*}

\subsection{Results}
Table \ref{semi} shows the results on semi-supervised node classification, it can be seen that SPCN-Conv has dominant performance on all datasets. In particular, with the sparse split, the performance of several methods varies significantly depending on the chosen training set. On the contrary, PCNet exhibits minimal fluctuations, indicating its stability. PCNet-$L$ can have better performance than SPCNet-$D$, showing that PCNet-$L$ has a better ability to learn information from very few data.

Table \ref{acc_res} shows our results compared to polynomial-based methods. It can be seen that our method achieves the best performance in most cases and is competitive on Actor. SPCNet's superior performance verifies that our proposed filter can better capture meaningful information on both homophilic and heterophilic graphs. Compared to the state-of-the-art graph filter with a learnable basis, OptBasisGNN, our method outperforms it in most cases. Compared to the most recent methods, BernNet and JacobiConv, our method is designed based on both a local and a global view, which has been shown to enhance the performance. Furthermore, SPCNet-$L$ also performs better than SPCNet-$D$ in most cases, which shows that the learnable $k$ can generally adapt to different datasets.

Table \ref{tab:result_heter} presents the results of the heterophilic methods. Our method performs better than recent heterophilic methods on 5 out of 7 datasets, which validates our cross-receptive filter. Thus, SPCNet can successfully unify homophilic and heterophilic datasets. Our method only performs poorly on Cornell; this may be due to Cornell having very few nodes and edges, making it difficult to extract meaningful information with our filter. We can also see that MLP performs better than traditional GNNs, which shows that attribute information is more meaningful than topology information on heterophilic datasets. Thus, the identity mapping is helpful. SPCNet-$D$ and SPCNet-$L$ produce close performance. 

\subsection{Robustness test}
In this section, we follow the experiments in \cite{EvenNet} to verify the robustness of our model. We apply non-targeted adversarial attacks to the structures, which are Meta-gradients (Meta) \cite{zugner_adversarial_2019} and MinMax \cite{xu2019topology} attacks. Meta attack regards the topology structure as hyperparameters and calculates the gradient of the attack loss function when backpropagating with differentiable model. MinMax attack is obtained by optimizing the convex relaxation with Boolean variables. They are trained to degrade the capability of a surrogate GNN model. Several changes, represented by the perturb ratio, are applied to the graph structure. The perturb ratio is set to 20$\%$ and there are 5 random splits in our experiment. These attacks also enlarge the homophily differences between training and testing graphs. 

Datasets include: Cora, Citeseer, and Acm \cite{wang2019heterogeneous}. The information of Acm can be seen in Table \ref{tab::datasets}. As for the baselines, five defense models that are robust to adversarial attacks are also included, which are: RobustGCN \cite{zhu2019robust}, GNN-SVD \cite{entezari2020all}, GNN-Jaccard \cite{wuadversarial}, GNNGuard \cite{zhang2020gnnguard}, and ProGNN \cite{jin2020graph}.

\textbf{Results}. The results with perturb ratio 0.20 are illustrated in Table \ref{rt20}. It can be seen that our method can dominate in nearly all cases. Most baselines face a dramatic decrease in performance. EvenNet is the most recent polynomial-based robust method, and our method surpasses it considerably in most datasets. In addition, we also test the performance with different perturb ratios. The results are shown in Fig. \ref{rt}. We also see the dominating performance of our model versus many purposely designed models. Thus, in addition to having high performance, our method can also be robust to different noises, making it applicable in the real world. This is consistent with our robustness analysis.



\section{Ablation Study}

To understand the effect of the dimensional transformation, we replace SPCNet's MLP with a single linear layer, which is denoted as SPCNet-Lin. From Table \ref{tab::abl1}, we can see that the performance of SPCNet-Lin and SPCNet-$D$ is very close (less than 1\%). This indicates that our filter has a strong fitting capability.

We also remove the identity mapping in our model and denote it as SPCNet-GNN. According to Table \ref{tab::abl1}, performance decreases in all cases, which verifies the importance of the original characteristic. However, the degradation is slight, and thus our proposed filter itself plays a key role.

\begin{table}[h]
\centering
\caption{Node classification results of PCNet and SPCNet with dense split (the per-epoch time/total training time (ms) ).}\label{tab::abl2}
\begin{center}
\resizebox{.5\textwidth}{!}{
\begin{tabular}{ l cccc}

\toprule
Datasets &PCNet &SPCNet-$D$ \\   
\midrule

Cora 
&88.41$_{\pm\text{0.66}}$ (3.74/850)
&\textbf{88.89$_{\pm\text{0.80}}$} (2.96/590)
\\
Citeseer
 
&77.50$_{\pm\text{1.06}}$ (3.50/740)
&\textbf{77.69$_{\pm\text{0.92}}$} (3.18/650)
 \\
Pubmed 

&89.51$_{\pm\text{0.28}}$ (3.88/830)
&\textbf{89.87$_{\pm\text{0.28}}$} (3.64/780)

\\
Texas 

&\textbf{88.11$_{\pm\text{2.17}}$} (3.84/900)
&87.84$_{\pm\text{2.43}}$ (3.15/640)

\\
Actor 

&37.80$_{\pm\text{0.64}}$ (3.55/1340)
&\textbf{37.83$_{\pm\text{0.72}}$} (3.19/670)

\\

Cornell 

&\textbf{82.16$_{\pm\text{2.70}} $} (3.80/780)
&80.81$_{\pm\text{2.98}} $ (3.34/700)

\\

Penn94
&82.93$_{\pm\text{0.18}} $ (4.32/2820)
&\textbf{84.75$_{\pm\text{0.27}}$} (4.09/2560)
\\
\bottomrule
\end{tabular}}
\end{center}
\end{table}

\section{Comparison with PCNet}
As a simplified version of PCNet \cite{li2024pc}, we compare with it to verify SPCNet's advantages.
We conduct experiments on node classification using semi-supervised node classification methodology. The results are displayed in Table \ref{semi2}. Though PCNet has more hyperparameters, SPCNet-$D$ outperforms it in 5 out of 6 datasets. With the exception of Cora, we observe at least  1\% improvement in performance.  We also conduct a comparison using dense split methodology, and the corresponding results are presented in Table \ref{tab::abl2}. Similarly to the sparse split scenario, SPCNet-$D$ exhibits superior performance on 4 of 6 data sets. This confirms that SPCNet generally works better than PCNet. The running time is also reported. It is evident that SPCNet is more efficient in all cases. These results unequivocally demonstrate the clear advantages of our proposed method over PCNet. 


\section{Conclusion}
To address heterophily problem in graph neural networks, we propose a novelty model named SPCNet. We design a cross-receptive filter that captures local and global graph structure information. Specifically, the global filter tackles heterophily by treating even- and odd-hop neighbors differently. We approximate it by the Poisson-Charlier polynomial to make it have linear complexity. Theoretical analysis proves its robustness. Comprehensive experiments verify its promissing  performance over state-of-the-art methods on real-world datasets. 

\bibliographystyle{IEEEtran}
\bibliography{PC}

\end{document}